\documentclass[10pt,twocolumn,letterpaper]{article}

 \usepackage{cvpr}              

\usepackage{graphicx}
\usepackage{amsmath}
\usepackage{amssymb}
\usepackage{booktabs}
\usepackage{amsthm}
\usepackage{multirow}
\usepackage{makecell}
\usepackage{colortbl}
\usepackage{array}
\usepackage{diagbox}
\usepackage{booktabs}
\usepackage{tikz}
\usepackage{fbox}
\usepackage{bm}
\usepackage{bbm}
\usepackage[misc]{ifsym}
\usepackage{caption,subcaption}
\usepackage{enumerate}
\usepackage{svg}
\usepackage[accsupp]{axessibility}

\definecolor{cvprblue}{rgb}{0.21,0.49,0.74}
\usepackage[pagebackref,breaklinks,colorlinks,citecolor=cvprblue]{hyperref}

\usepackage[capitalize]{cleveref}
\crefname{section}{Sec.}{Secs.}
\Crefname{section}{Section}{Sections}
\Crefname{table}{Table}{Tables}
\crefname{table}{Tab.}{Tabs.}

\newtheorem{proposition}{Proposition}

\def\paperID{6971} 
\def\confName{CVPR}
\def\confYear{2024}

\begin{document}
\title{Correspondence-Free Non-Rigid Point Set Registration Using \\Unsupervised Clustering Analysis}

\newcommand*{\affaddr}[1]{#1}
\newcommand*{\affmark}[1][*]{\textsuperscript{#1}}
\newcommand*{\email}[1]{\small{\texttt{#1}}}
\author{
Mingyang Zhao$^{1}$ \quad
Jingen Jiang$^{2}$ \quad
Lei Ma$^{3*}$ \quad
Shiqing Xin$^{2}$ \quad
Gaofeng Meng$^{4,5}$ \quad
Dong-Ming Yan$^{4,5*}$ \quad
\\
\affaddr{$^{1}$CAIR, HKISI, CAS}\quad
\affaddr{$^{2}$Shandong University}\quad
\affaddr{$^{3}$Peking University}\quad
\affaddr{$^{4}$MAIS, CASIA}\quad
\affaddr{$^{5}$UCAS}\quad
\\
}

\twocolumn[{
\renewcommand\twocolumn[1][]{#1}
\maketitle
\begin{center}
\captionsetup{type=figure}
\vspace{-0.9cm}
\includegraphics[width=0.87\textwidth]{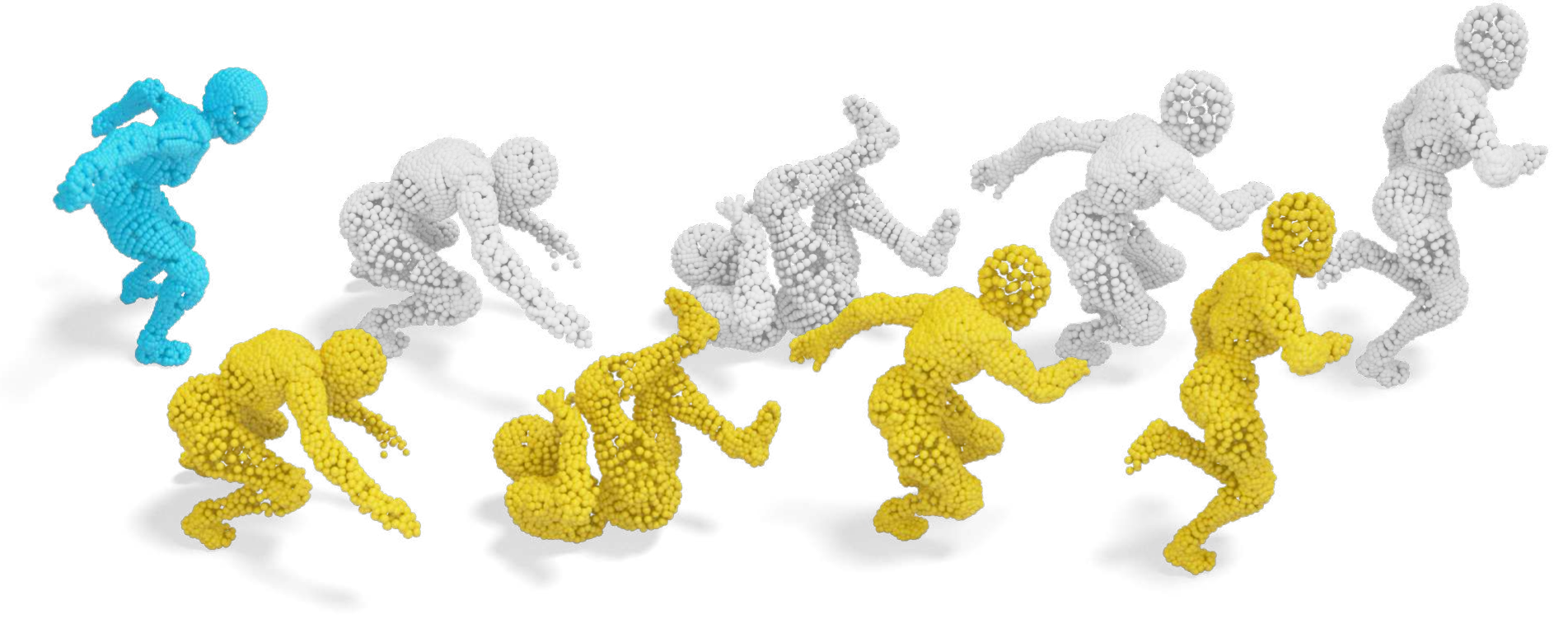}
\vskip -0.3cm
\captionof{figure}{Non-rigid registration on 3D point sets. The blue and gray models represent the source and target point clouds, respectively, while the yellow models are our registration results. Our method achieves successful registrations even for shapes with challenging deformations.}
\label{fig:teaser}
\end{center}
}]

\maketitle

\renewcommand{\thefootnote}
{\fnsymbol{footnote}}
\footnotetext[1]{Corresponding Authors.}

\begin{abstract}
This paper presents a novel non-rigid point set registration method that is inspired by unsupervised clustering analysis. Unlike previous approaches that treat the source and target point sets as separate entities, we develop a holistic framework where they are formulated as clustering centroids and clustering members, separately. We then adopt Tikhonov regularization with an $\ell_1$-induced Laplacian kernel instead of the commonly used Gaussian kernel to ensure smooth and more robust displacement fields. Our formulation delivers closed-form solutions, theoretical guarantees, independence from dimensions, and the ability to handle large deformations. Subsequently, we introduce a clustering-improved Nystr{\"o}m method to effectively reduce the computational complexity and storage of the Gram matrix to linear, while providing a rigorous bound for the low-rank approximation. Our  method achieves high accuracy results across various scenarios and surpasses competitors by a significant margin, particularly on shapes with substantial deformations. Additionally, we demonstrate the versatility of our method in challenging tasks such as shape transfer and medical registration. \href{https://github.com/zikai1/CVPR24_PointSetReg}{[Code release]}  
\end{abstract}

\section{Introduction}
\label{sec:intro}
Non-rigid point set registration is to 
optimize a non-linear displacement field that accurately aligns one geometric shape with another. Due to its fundamental importance, non-rigid registration plays a dominant role in a wide range of applications, such as scene reconstruction~\cite{zollhofer2014real,zhao2022graphreg,park2021nerfies}, pose tracking~\cite{newcombe2015dynamicfusion}, animation~\cite{siarohin2023unsupervised}, deformable shape manipulation and editing~\cite{tang2022neural}, and so on.

However, given two point sets, one acting as the source and the other as the target, non-rigid registration presents a highly ill-posed and much more complex challenge compared to the rigid counterpart. This increased complexity is primarily attributed to the additional freedom of deformations allowed in non-rigid registration, especially when dealing with shapes that exhibit large deformations (\cref{fig:teaser}).

To enhance the registration quality for shapes undergoing large deformations, numerous pioneering methods have been actively researched. Rather than directly optimizing the registration process, these methods usually employ a \emph{two-step} approach~\cite{ma2018nonrigid,Yao_2020_Fast_RNRR,yao2023fast,qin2023deep}. First, they perform \emph{shape matching} by identifying corresponding points between the source and target shapes without considering geometry deformations. Then, they estimate the alignment transformation based on the established correspondences via off-the-shelf registration techniques. While there has been significant attention and research dedicated to the initial shape matching stage, the exploration of \emph{direct registration} methods for handling large deformations, without relying on shape matching, is comparatively limited and poses substantial challenges~\cite{hirose2020bayesian}.

In this work, we address the problem of non-rigid point set registration without correspondences, with a specific emphasis on point sets exhibiting large deformations. To overcome this challenge, we present a fresh perspective and introduce a novel method. Our approach reformulates the non-rigid deformation process as an \emph{unsupervised clustering problem} within the context of machine learning. Unlike previous approaches that treat the two point sets as \emph{separate} entities, we consider them as \emph{integral} parts of a whole.

Concretely, we designate the source point set as the clustering centroids, while the target one as the clustering samples. This \emph{holistic} treatment enables us to leverage the interplay between these two sets. Then the dynamic optimization and update of the clustering centroids correspond to the underlying deformation of the source shape. We highlight the advantages of our novel registration function, which is built on clustering analysis, from both information theory and convex optimization perspectives. Furthermore, we provide \emph{closed-form} solutions to our objective function during each iteration, which enables fast and efficient implementations. We introduce a sparsity-induced \emph{Laplacian kernel} ($\ell_1$-norm) in the Tikhonov regularization to ensure that the displacement field of clustering centroids remains as smooth as possible. This differs from the commonly used Gaussian kernel and exhibits higher robustness, as demonstrated by experimental results. Additionally, we leverage clustering analysis to adopt the improved Nyström low-rank approximation~\cite{zhang2008improved}, which reduces the computational complexity and storage requirements of the Gram matrix to linear. Meanwhile, we give a rigorous proof of the approximation error bound associated with the Laplacian kernel.

Our method is independent of spatial dimensions, allowing us to evaluate and compare its performance in both 2D and 3D settings. The experimental results demonstrate the superiority of our method compared to baselines by a large margin. This is particularly evident in scenarios involving large deformations, such as shape transfer and medical data registration.

Our contributions can be summarized as follows: 
\begin{itemize}
    \item We propose a novel and correspondence-free method for non-rigid point set registration, utilizing unsupervised clustering analysis. The method achieves impressive results across various settings and mitigates the challenge without explicit correspondences.
    \item We incorporate the Laplacian kernel function for robust displacement regularization and provide a rigorous theoretical analysis to prove the approximation error bound of the Nystr{\"o}m low-rank method.
    \item Our method is dimension-independent, offering closed-form solutions during optimization, and significantly improves performance in handling large deformations.
\end{itemize}

\vspace{-0.3cm}
\section{Related Work}
We review the work that is closely aligned with ours. Readers are directed to ~\cite{tam2012registration,deng2022survey} for comprehensive studies.
\vspace{-0.3cm}
\paragraph{Non-rigid registration.} Differing from shape matching that focuses on finding inlier correspondences, non-rigid registration aims to optimize the displacement field. Various pioneering algorithms employ an optimization paradigm that minimizes both the data and penalty terms simultaneously. Amberg~\etal~\cite{amberg2007optimal} extended the rigid  iterative closest point algorithm~\cite{besl1992method} to non-rigid settings, while Yao~\etal~\cite{Yao_2020_Fast_RNRR,yao2023fast} recently improved non-rigid ICP regarding both accuracy and efficiency through deformation graph optimization. \emph{Coherent Point Drift} (CPD)~\cite{myronenko2010point} and GMM~\cite{jian2010robust}  developed probabilistic frameworks by minimizing the negative logarithm likelihood function to enhance the robustness for non-rigid point set registration. Ma~\etal~\cite{ma2018nonrigid} further incorporated the shape context descriptor~\cite{belongie2002shape} to establish shape correspondences for better 2D registration. Hirose~\cite{hirose2020bayesian, hirose2022geodesic} recently formulated CPD in a Bayesian setting, which effectively overcomes CPD's limitations and delivers impressive results.

With the advancement of deep learning, neural network-based methods have also been proposed for non-rigid point set registration~\cite{huang2022multiway,li2022lepard,qin2023deep}. Most of them utilize neural networks to extract features for point correspondences and then apply classical methods such as non-rigid ICP for registration. Instead of focusing on shape matching and heavily rely on data annotations, our method is unsupervised and reasons from a case-by-case geometric perspective. This allows us to achieve faithful registrations that are more generalizable to unknown categories.

\paragraph{Deformation representation.} The representation of the deformation field is a key component in non-rigid registration. Several existing works are based on thin plate spline functions~\cite{chui2003new,jian2010robust,santa2013correspondence}, which can be viewed as a regularization of the second-order derivatives of the transformations~\cite{myronenko2010point}. Another line of researches utilize kernel functions or a reproducing kernel Hilbert space to describe the deformation field~\cite{myronenko2010point,ma2015non,ma2015robust,tang2018framework}. However, many of these methods are limited to the Gaussian kernel due to the reliance on fast Gauss transform ~\cite{greengard1991fast}. Recently, the \emph{Multi-Layer Perception} (MLP) network has been employed to represent the deformation field by mapping input coordinates to signal values~\cite{li2021neural,park2021nerfies,li2022non} and the deformation degree is controlled by frequencies. These methods have shown promising results in dynamical reconstruction and scene flow estimation, which are typically considered less challenging tasks compared to dealing with large deformations.

\section{Preliminaries on Clustering Analysis}\label{preliminary}
As one of the representative unsupervised learning frameworks, clustering analysis plays a fundamental role in various scientific research domains~\cite{xu2005survey}. The pioneering work~\cite{liao2020point,zhao2023accurate} explored clustering metrics for rigid point cloud registration. In contrast, we distinguish ourselves by addressing a more challenging non-rigid problem, which we have completely reformulated as a clustering process with a different objective function. We present a concise overview on two commonly used clustering approaches: \emph{fuzzy clustering} and \emph{Elkan k-means clustering} analysis.

\subsection{Fuzzy Clustering Analysis} Given a dataset $\mathbf{X}=\{\bm{x}_i\in \mathbb{R}^n\}_{i=1}^M$, fuzzy clustering analysis solves the following problem: 
\begin{small}
\begin{equation}
\min_{\mathbf{U},\mathbf{V}}\sum_{j\!=\!1}^C\sum_{i\!=\!1}^M(u_{ij})^r||\bm{x}_i\!-\!\bm{v}_j||_2^2, s.t.\sum_{j\!=\!1}^Cu_{ij}\!=\!1,u_{ij}\in[0, 1],
\label{eq:fcm}
\end{equation}
\end{small}where $\mathbf{U}=[u_{ij}]_{M\times C}\in\mathbb{R}^{M\times C}$ is the \emph{fuzzy membership degree matrix},  $\mathbf{V}=\{\bm{v}_j\in \mathbb{R}^n\}_{j=1}^C$ is the set of \emph{clustering centroids} consisting of $C\in\mathbb{Z}_{+}$ classes, and $r\in[1,+\infty)$ is the \emph{fuzzy factor}, which controls the clustering fuzziness. To enhance the clustering performance on unbalanced datasets, Miyamoto~\etal\cite{miyamoto2008algorithms} proposed the inclusion of \emph{cluster size controlling variables} $\bm{\alpha}=[\alpha_1, \cdots, \alpha_C]\in\mathbb{R}^{C}$ in ~\cref{eq:fcm}, and thus classes with more samples may lead to higher fuzzy membership degree. Since Euclidean distance-based clustering algorithms are primarily suitable for spherical data, \emph{Mahalanobis distance} is latter introduced to generalize the fuzzy clustering analysis to accommodate ellipsoidal structures~\cite{gustafson1979fuzzy}. Recently, \cite{chen2023ell} combined the merits of previous fuzzy clustering approaches and developed a novel clustering framework based on the $\ell_{2,p}$ norm, which achieves appealing results on a set of clustering analysis tasks: 
\begin{equation}
\begin{small}
\begin{gathered}
\min_{\mathbf{U},\mathbf{V},\mathbf{\Sigma},\bm{\alpha}}\sum_{j\!=\!1}^C\sum_{i\!=\!1}^M{u}_{ij}||\mathbf{\Sigma}_j^{-\frac{1}{2}}(\bm{x}_i\!-\!\bm{v}_j)||_2^p
\!+\!{u}_{ij}\text{log}|\mathbf{\Sigma}_j|\!+\!\lambda {u}_{ij}\text{log}\frac{{u}_{ij}}{\alpha_j},\\s.t.~ |\mathbf{\Sigma}_j|=\theta_j, \sum_{j=1}^C{u}_{ij}=1, \sum_{j=1}^C\alpha_{j}=1, u_{ij}, \alpha_j\in[0, 1]. 
\end{gathered}
\label{eq:loss_cluster}
\end{small}
\end{equation}where $\lambda\in\mathbb{R}^{+}$ is a regularization parameter, and  $\mathbf{\Sigma}_j\in \mathbb{S}^{n}_{++}\triangleq\{\mathbf{A}\in\mathbb{R}^{n\times n}|\bm{x}^{T}\mathbf{A}\bm{x}>0, \forall \bm{x}\in\mathbb{R}^n\}$ denotes the covariance matrix of the $j$-th class, with the corresponding determinant equivalent to $|\mathbf{\Sigma}_j|\in\mathbb{R}$. We explore the application of this clustering analysis framework to non-rigid point set registration and demonstrate its superior performance over previous registration approaches.

\subsection{Elkan $k$-Means Clustering} \label{sec:k-means}
In contrast to fuzzy clustering analysis, the $k$-means algorithm~\cite{macqueen1967some} has emerged as one of the most widely used clustering methods due to its simplicity. {Elkan $k$-means clustering} further introduced the \emph{triangle inequality} into the $k$-means framework to avoid unnecessary distance calculations, which dramatically speeds up the primary $k$-means clustering process. More details of Elkan $k$-means clustering can be found in~\cite{elkan2003using}.

\section{Proposed Method}\label{method}  
\paragraph{Problem formulation.} Given two point sets $\mathbf{X}=\{\bm{x}_i\in \mathbb{R}^n\}_{i=1}^M$ and $\mathbf{Y}=\{\bm{y}_j\in \mathbb{R}^n\}_{j=1}^N$, where $\mathbf{X}$ and $\mathbf{Y}$ are named as the target and the source, separately, the objective of non-rigid point set registration is to find the optimal deformation map $\mathcal{T}$ that minimizes the shape deviation between $\mathcal{T}(\mathbf{Y})\triangleq\mathbf{Y}+\mathbf{\nu}(\mathbf{Y})$ and $\mathbf{X}$, where $\mathbf{\nu}$ represents the displacement filed acting on each source point $\bm{y}_j$.

\subsection{Clustering-Induced Non-Rigid Registration}
\paragraph{Observations.} We notice that during the clustering process, the spatial position of clustering centroids $\mathbf{V}$ are dynamically updated until the distance between the centroids and their members is minimized. This dynamic process bears resemblance to the iterative update of each source point $\mathcal{T}(\bm{y}_j)$. Inspired by this, we propose to formulate non-rigid registration as an unsupervised clustering process. We consider $\mathbf{Y}$ as the clustering centroids and $\mathbf{X}$ as the clustering members. We customize \cref{eq:loss_cluster} to optimize the overall clustering loss by
\begin{small}
\begin{equation}
\begin{gathered}
\min F(\mathbf{U},\bm{\alpha},\mathbf{\Sigma},\nu)\begin{aligned}=\sum_{j=1}^C\sum_{i=1}^M{u}_{ij}||\mathbf{\Sigma}_j^{-\frac{1}{2}}(\bm{x}_i\!-\!(\bm{y}_j\!+\!\nu(\bm{y}_j)))||_2^2\end{aligned} \\
\begin{aligned}\!+\!{u}_{ij}\text{log}|\mathbf{\Sigma}_j|\!+\!\lambda {u}_{ij}\text{log}\frac{{u}_{ij}}{\alpha_j},\end{aligned} \\
s.t.~|\mathbf{\Sigma}_j|=\theta_j, \sum_{j=1}^C{u}_{ij}=1, \sum_{j=1}^C\alpha_{j}=1, u_{ij}, \alpha_j\in[0,1]. 
\end{gathered}
\label{eq:loss_reg}
\end{equation} 
\end{small}Here we set $p=2$ to ease the computation, which also ensures closed-form solutions as derived in the following.

\paragraph{Regularization.} As in~\cite{myronenko2010point,hirose2020bayesian}, we incorporate \emph{Tikhonov regularization}~\cite{girosi1995regularization} to promote smoothness in the displacement field of clustering centroids. Thus, our objective function is optimized to find the optimal locations of clustering centroids as follows:
\begin{small}
\begin{equation}
\min F(\mathbf{U},\bm{\alpha}, \mathbf{\Sigma},\nu)+\zeta \mathcal{R}(\nu),
\label{eq:final_loss}
\end{equation}
\end{small}where $\zeta$ is a trade-off parameter. $\mathcal{R}(\cdot)$ is an operator that penalizes the high-frequency component of $\nu$ if we consider it in the Fourier domain, \ie,
\begin{small}
\begin{equation}
\mathcal{R}(\nu)=\int_{\mathbf{R}^n}d\mathbf{s}\frac{||\tilde{\nu}(\mathbf{s})||_2^2}{\tilde{K}(\mathbf{s})}.
\label{eq:regularization}
\end{equation}
\end{small}$K(\mathbf{s})$ is a kernel function regarding the frequency variable $\mathbf{s}$, and $\tilde{f}$ indicates the Fourier transform of the function $f$. 
\subsection{Virtues of the Newly-Defined Function}
We provide a theoretical analysis of \cref{eq:loss_reg} from both information theory and optimization perspectives. This analysis allows us to highlight the virtues of our newly introduced loss function for non-rigid point set registration.
\paragraph{Information theory view.} We re-write $F(\mathbf{U}, \bm{\alpha}, \mathbf{\Sigma},\nu)$ as
\begin{small}
\begin{equation*}
\sum_{j=1,i=1}^{C,M}{u}_{ij}||\mathbf{\Sigma}_j^{-\frac{1}{2}}(\bm{x}_i-(\bm{y}_j+\nu(\bm{y}_j)))||_2^2
+{u}_{ij}\text{log}(\frac{|\mathbf{\Sigma}_j|}{{\alpha_j^{\lambda}}})-\lambda H(\mathbf{U}),
\label{eq:loss_cluster_info}
\end{equation*}
\end{small}where $H(\mathbf{U})=-\sum_{j=1}^C\sum_{i=1}^Mu_{ij}\log(u_{ij})$ is the entropy of $\mathbf{U}$. From the perspective of information theory, this \emph{entropy regularization} term serves to push $F(\mathbf{U},\bm{\alpha},\mathbf{\Sigma},\nu)$ towards $\mathbf{U}$ with a uniform distribution that makes $H(\mathbf{U})$ the maximal and thus drags $F(\mathbf{U},\alpha,\mathbf{\Sigma},\nu)$ away from the sparse $\mathbf{U}$. This not only enhances the smoothness of the feasible set, but also improves the computational stability during optimization, \ie, avoiding $\lim\limits_{u_{ij}\rightarrow0}\log(u_{ij})=-\infty$~\cite{cuturi2013sinkhorn}.

\paragraph{Optimization view.} Alternatively, from an optimization point of view, $u_{ij}\log(u_{ij})$ is a convex function in terms of $u_{ij}$ with $\lambda$ controlling the degree of convexity. Moreover, $u_{ij}\log(u_{ij})$ acts as a \emph{barrier function} that restricts $u_{ij}$ to the range of $[0, 1]$ and prevents it from taking values outside this range~\cite{yuille1994statistical}.

\subsection{Closed-Form Solutions} 
Our method enables closed-form solutions for each variable during the optimization step as derived in the following.

\paragraph{Update of $\mathbf{U}$.} We fix $\bm{\alpha}, \mathbf{\Sigma}, \nu$ and update $\mathbf{U}$, which becomes a convex optimization problem. Utilizing the Lagrangian multiplier and ignoring parameters that are irrelevant to $\mathbf{U}$, we obtain
\begin{small}
\begin{equation*}
\begin{gathered}
\mathcal{L}(\mathbf{U},\bm{\beta})\begin{aligned}=\sum_{j=1}^C\sum_{i=1}^M{u}_{ij}||\mathbf{\Sigma}_j^{-\frac{1}{2}}(\bm{x}_i-(\bm{y}_j+\nu(\bm{y}_j)))||_2^2\end{aligned} \\
\begin{aligned}+{u}_{ij}\text{log}|\mathbf{\Sigma}_j|+\lambda {u}_{ij}\text{log}\frac{{u}_{ij}}{\alpha_j}+\sum_{i=1}^M\beta_i(\sum_{j=1}^C{u}_{ij}-1),\end{aligned}
\end{gathered}
\label{eq:loss_U}
\end{equation*}
\end{small}where $\bm{\beta}=\{\beta_i\in\mathbb{R}\}_{i=1}^M$ are the set of Lagrangian multipliers. By equating $\frac{\partial\mathcal{L}}{\partial\mathbf{U}}=0$, we have 
\begin{small}
\begin{equation}
\begin{aligned}\mathbf{U}=(\mathrm{diag}(\mathbf{A}\mathbf{1}_{C}))^{-1}\mathbf{A}
\\
\end{aligned}
\label{eq:U}
\end{equation}
\end{small}Here $\mathbf{A}={\mathrm{exp}(-\mathbf{D}/\lambda)\mathrm{diag}(\bm{\alpha}\odot|\mathbf{\Sigma}|)}$, $\mathbf{D}=[d_{ij}]_{M\times C}\in \mathbb{R}^{M\times C}$ is a squared Euclidean distance matrix with $d_{ij}=\|\mathbf{\Sigma}_j^{-\frac{1}{2}}(\bm{x}_i-(\bm{y}_j+\nu(\bm{y}_j)))\|_2^2$, $\mathrm{exp}(\cdot)$ is the element-wise exponential operator of matrices, $\mathrm{diag}(\mathbf{z})$ is an operator that creates a square diagonal matrix with the vector $\mathbf{z}$ on its main diagonal, and $|\mathbf{\Sigma}|=[|\mathbf{\Sigma}_1|, \cdots, |\mathbf{\Sigma}_C|]^T\in \mathbb{R}^C$. $\mathbf{1}_C$ is the $C$-dimensional vector of all ones, and $\odot$ represents the element-wise Hadamard product of two matrices or vectors.

\paragraph{Update of $\bm{\alpha}$.} Likewise, the closed-form solution with respect to $\bm{\alpha}$ is $\bm{\alpha}=\frac{1}{M}\mathbf{U}^{T}\mathbf{1}_{M}$, which formally quantifies the clustering size for each class.

\paragraph{Update of $\mathbf{\Sigma}$.} For simplicity, we relax each clustering centroid's covariance matrix to be isotropic, \ie,  $\mathbf{\Sigma}_j=\sigma^2\mathbf{I}$, where $\mathbf{I}\in \mathbb{R}^{n\times n}$ is the identity matrix. This ensures a closed-form solution to variance 
$\sigma^2$:
\begin{small}
\begin{equation*}
	\sigma^2=\frac{\operatorname{tr}(\mathbf{X}^T\mathrm{diag}(\mathbf{U}^T\mathbf{1}_{M})\mathbf{X}-(2(\mathbf{U}\mathbf{X})^T+\mathbf{T}^T\mathrm{diag}(\mathbf{U}\mathbf{1}_{C}))\mathbf{T})}{n\times M},
\label{eq:sigma}
\end{equation*}
\end{small} where $\operatorname{tr}(\cdot)$ is the matrix trace operator.

\begin{figure}[t]
\centering
\includegraphics[width=0.236\textwidth]{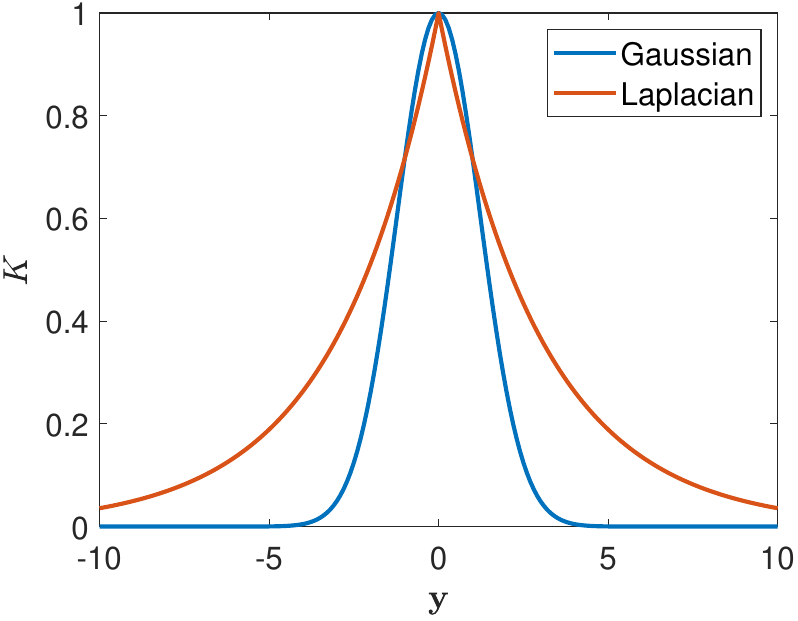}
\includegraphics[width=0.236\textwidth]{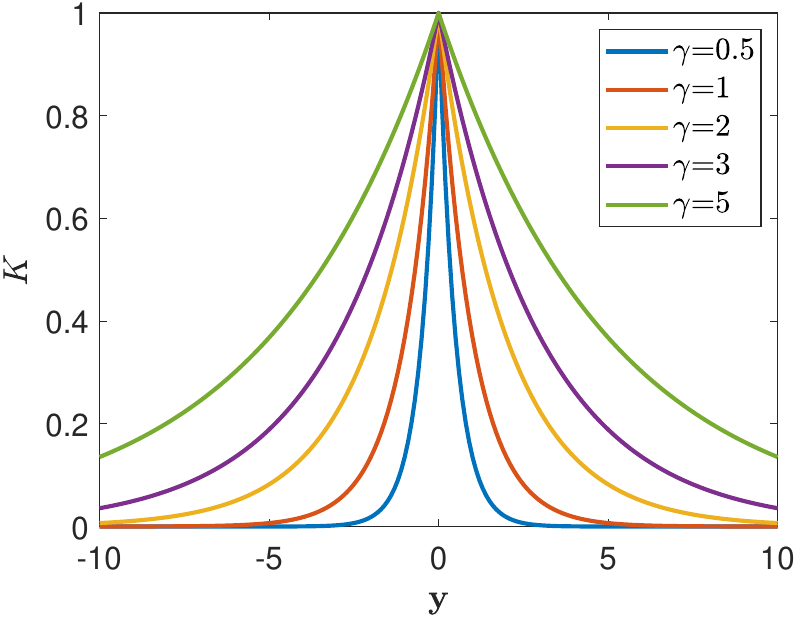}
\vskip -0.2cm
\caption{Left: Comparison between the Gaussian kernel and Laplacian kernel with the same bandwidth $\gamma=3$, where the latter delivers considerably thicker lower and upper tails, indicating higher robustness. Right: Laplacian kernel with different $\gamma$.}
\label{fig:gauss_laplace}
\vskip -0.7cm
\end{figure}
\paragraph{Update of ${\nu}$.} By leveraging the \emph{Riesz’s representation theorem}~\cite{chen2002different}, the closed-form solution to the regularization term $\nu$ can be expressed as
\begin{equation}
\centering
\nu(\bm{y})=\sum_{j=1}^Cc_jK(\bm{y},\bm{y}_j)+\sum_{\eta=1}^Nd_{\eta}\psi_{\eta}(\nu),
\label{eq:nu}	
\end{equation}where $\{c_j\in\mathbb{R}\}_{j=1}^{C}$ are the coefficient scalars, $K(\cdot,\cdot)$ is the kernel function defined in ~\cref{eq:regularization}, and $\{\psi_{\eta}\}_{\eta=1}^N$ represent a set of basis in the $N$-dimensional null space of $\mathcal{R}(\nu)$, which is typically composed by a set of polynomials for most choices of the stabilizer $\mathcal{R}(\nu)$.

In contrast to previous approaches that commonly utilize a \emph{Gaussian Radial Basis Function} (RBF)~\cite{myronenko2010point,ma2018nonrigid}, we adopt the sparsity-induced \emph{Laplacian kernel} with the robust $\ell_1$-norm to characterize the displacement field $\nu$, \ie,
\begin{eqnarray}
\centering
K(\bm{y}_i, \bm{y}_j)=\exp(-\gamma\|\bm{y}_i-\bm{y}_j\|_1), \quad \gamma>0
\label{eq:kernel}
\end{eqnarray} in which $\|\bm{y}_i-\bm{y}_j\|_1$ is the \emph{Manhattan distance} between the two input vectors. Compared to the RBF kernel, the Laplacian kernel exhibits stronger robustness due to its considerably thicker tails, as illustrated in~\cref{fig:gauss_laplace}. We also validate this conclusion through subsequent experiments.

Since the Laplacian kernel is positive definite, we obtain $\psi_{\eta}\equiv0$~\cite{girosi1995regularization}. By evaluating $\nu(\bm{y})$ at $\mathbf{Y}=\{\bm{y}_j\in \mathbb{R}^n\}_{j=1}^C$, following~\cite{myronenko2010point,hirose2020bayesian}, the coefficient vector $\mathbf{c}=[c_1, c_2, \cdots, c_C]^T\in \mathbb{R}^C$ is recovered from the following linear system:
\begin{small}
\begin{equation}
\centering
\mathbf{c}=(\mathbf{L}+\zeta\sigma^2\mathrm{diag}(\mathbf{U}\mathbf{1}_{C})^{-1})^{-1}(\mathrm{diag}(\mathbf{U}\mathbf{1}_C)^{-1}\mathbf{U}\mathbf{X}-\mathbf{Y}),
\label{eq:C}
\end{equation}
\end{small}where $\mathbf{L}$ is the Gram matrix
with $l_{ij}=K(\bm{y}_i, \bm{y}_j)$. Therefore, the newly deformed shape $\mathbf{T}$ from the source point set $\mathbf{Y}$ becomes $\mathbf{T}=\mathcal{T}(\mathbf{Y})=\mathbf{Y}+\mathbf{L}\mathbf{c}$.

\subsection{Improved Nystr{\"o}m Low-Rank Approximation}
The matrix inverse operation in~\cref{eq:C} leads to a computational complexity of $O(C^3)$ and a memory requirement of $O(C^2)$. Previous approaches often employ the \emph{fast Gauss transform} (FGT)~\cite{greengard1991fast} to reduce memory usage and accelerate computation. However, FGT is merely limited to the Gaussian kernel. To circumvent this issue, BCPD~\cite{hirose2020bayesian} combined the Nystr{\"o}m method~\cite{christopher2001using} and the KD tree search~\cite{bentley1975multidimensional} for acceleration. However, there are still two major issues that remain unresolved. (1) Due to the random sampling scheme used in BCPD, it is unclear how effective the Nystr{\"o}m approximation performs. (2) In order to address convergence issues when $\sigma^2$ becomes small, BCPD need to switch from Nystr{\"o}m approximation to KD tree search. This transition may affect the optimization trajectory.

To overcome these challenges, we opt to use \emph{clustering analysis} instead of random sampling. Concretely, we first employ the fast Elkan $k$-means algorithm (\cref{sec:k-means}) to partition $\mathbf{Y}$ into $C'$ disjoint clusters $\mathbf{P}_i\subset \mathbf{Y}$, with the corresponding clustering centroids as {$\{\bm{z}_i\in\mathbb{R}^n\}_{i=1}^{C'}$} ($C'\ll C$). Then, we adopt the improved Nystr{\"o}m approximation~\cite{zhang2008improved} for efficient and consistent optimization:
\begin{equation}
    \centering
\mathbf{L}\approx\mathbf{E}\mathbf{W}^{-1}\mathbf{E}^{T},
\end{equation}where $\mathbf{E}=[e_{ij}]\in \mathbb{R}^{C\times C'}$ and $\mathbf{W}=[w_{ij}]\in \mathbb{R}^{C'\times C'}$ are the low-rank Laplacian kernel matrices, with elements $e_{ij}=K(\bm{y}_i, \bm{z}_j)$ and $w_{ij}=K(\bm{z}_i, \bm{z}_j)$.

By incorporating clustering analysis, we achieve two key benefits: (1) rigorously proving the error bound of the Nystr{\"o}m approximation for our utilized Laplacian kernel, and (2) as demonstrated through experiments, providing encouraging results for non-rigid point set registration without compromising the optimization trajectory.

\begin{proposition}
	\label{prop:1}
The low-rank approximation error $\epsilon=\|\mathbf{L}-\mathbf{E}\mathbf{W}^{-1}\mathbf{E}^{T}\|_{F}$ in terms of the Laplacian kernel is bounded by
\begin{equation}
\epsilon\leq4\sqrt{2}T^{3/2}\gamma\sqrt{C'q}+2C'\gamma^2 Tq\|W^{\boldsymbol{-}1}\|_F,
\end{equation}where $\|\cdot\|_{F}$ is the matrix Frobenious norm, $T=\max_i|\mathbf{P}_i|$, $q=\sum_{j=1}^{C}\|\bm{y}_j-\bm{z}_{c'(j)}\|_2^2$ is the clustering quantization error with $c'(j)={\rm argmin}_{i=1, \cdots, C'}\|\bm{y}_j-\bm{z}_i\|_2$, and $\gamma$ is the Laplacian kernel bandwidth defined in Eq.~(\ref{eq:kernel}). 
\end{proposition}
\begin{proof}
	Please see the \emph{Supplementary Material}.
\end{proof}

\section{Experimental Results}
We perform extensive experiments to demonstrate the performance of the proposed method and compare it with state-of-the-art approaches from both 2D and 3D categories. 
\paragraph{Implementation details.} Given a pair of point sets, for better numerical stability, we first perform shape normalization to make them follow the standard normal distribution. However, the registration evaluation is still based on the original inputs through denormalization. The Laplacian kernel bandwidth $\gamma$ is set to $2$ by default, and the number of clustering centroids in Elkan $k$-means equals to $0.3C$ for better trade-off between registration accuracy and efficiency. During optimization, we fix the two weight coefficients $\lambda=0.5$ and $\zeta=0.1$, which deliver impressive performance across various scenes. Our algorithm is implemented in MATLAB, on a computer running AMD Core Ryzen 5 3600XT (3.8GHz). We leverage publicly available implementations of baseline approaches for assessment, with their parameters either fine-tuned by ourselves or fixed by the original authors to achieve their best results.

\paragraph{Evaluation criteria.} As in~\cite{hirose2020bayesian}, we adopt the \emph{Root Mean Squared Error} (RMSE) to quantitatively assess the registration accuracy. For point sets with known ground-truth correspondences, we compute the squared distance between corresponding points directly. However, for point sets without annotated correspondences, such as distinct types of geometries, we identify the corresponding point pairs through the nearest neighbor search. Accordingly, the RMSE is defined as:
\begin{small}
\begin{equation*}
{{	\centering
	\operatorname{RMSE}(\mathcal{T}(\mathbf{Y}),{\mathbf{X}})\!=\!\sqrt{\operatorname{Tr}\{\left(\mathcal{T}(\mathbf{Y})\!-\!{\mathbf{X}}\right)^T(\mathcal{T}(\mathbf{Y})\!-\!{\mathbf{X}})\}/M}}},
\end{equation*}
\end{small}where $\mathcal{T}(\mathbf{Y})$ and ${\mathbf{X}}$ are the deformed and the target point sets, respectively. 

\subsection{2D Non-Rigid Point Set Registration}\label{sec:hand}
For 2D non-rigid point set registration, we utilize the benchmark IMM hand dataset~\cite{stegmann2002brief} for evaluation. This dataset encompasses 40 real images, showing the left hands of four distinct subjects, and each contains 10 images. As illustrated in Fig.~\ref{fig:hand_input}, the hand shape is described through 56 key points extracted from the contour of the hand. We employ the first pose from each group of hands as our target point set, while the remaining poses of the same subject serve as the source point sets. The quantitative comparison results with state-of-the-art 2D registration approaches including MR-RPM~\cite{ma2018nonrigid}, BCPD~\cite{hirose2020bayesian}, GMM~\cite{jian2010robust}, and ZAC~\cite{wang2020zero} are reported in~\cref{tab:hand}. We report the 
average RMSE for each subject along with the average registration timing of each method. As observed, our method consistently outperforms the comparative approaches with higher registration accuracy and efficiency across all subjects. Although without the need for constructing the initial point correspondences, like shape context~\cite{belongie2002shape} used in MR-RPM, our method still delivers RMSE that is orders of magnitude lower than that of most competitors, highlighting its compelling advantages. The qualitative comparison results regarding the inputs in the third row of \cref{fig:hand_input} are presented in Fig.~\ref{fig:hand_reg}.  
\vskip -0.3cm

\begin{figure}[t]
    \centering
\includegraphics[width=\linewidth]{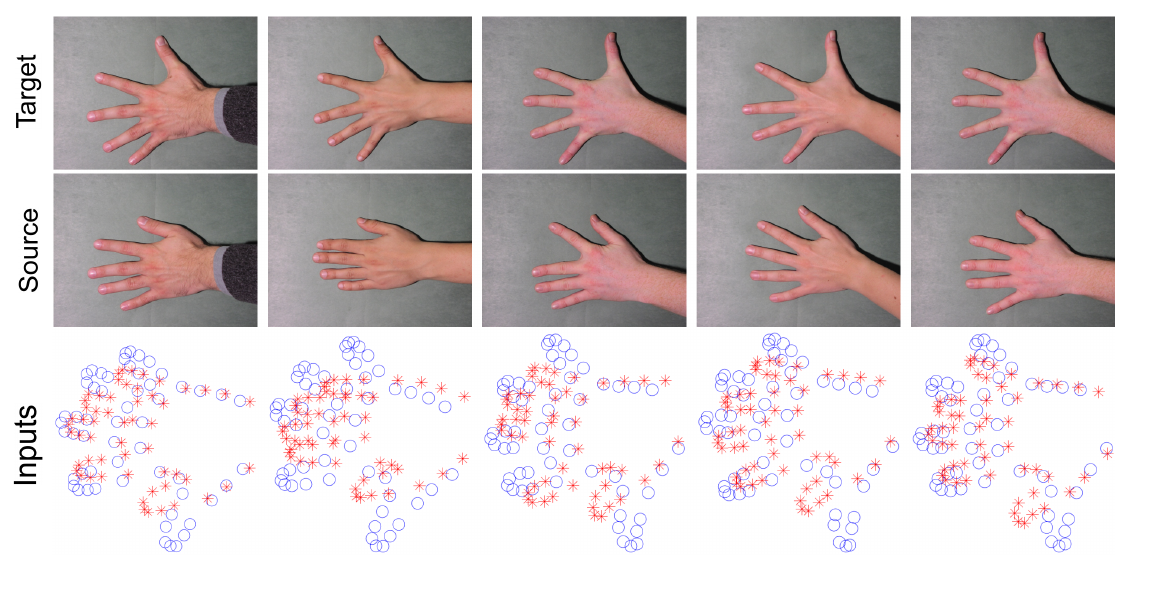}
\vskip -0.1cm
    \caption{Test samples from the 2D hand pose dataset with different subjects. The red {\color{red}{$\ast$}} and the blue {\color{blue}{$\circ$}} indicate the source and the target point sets, respectively.}
\label{fig:hand_input}
\vskip -0.5cm
\end{figure}

\paragraph{Robustness.} We further investigate the robustness of the designed method against external disturbances including noise and occlusion. We add a set of Gaussian noise with zero mean and varying standard deviations $\sigma\in[0.01, 0.06]$ to all of the source point sets defined in the above section. Additionally, we randomly erase several points, around $3\%\sim20\%$ of the source, to construct a range of occlusion geometries. \cref{fig:robustness} summarizes the average RMSE values across all subjects. It can be observed that our method still achieves the highest or comparable registration accuracy on all settings, highlighting its stability and robustness. Qualitative comparison results are presented in the \emph{Supplementary Material}.

\begin{table}[!htbp]
	\centering
	\caption{{Quantitative comparisons on the 2D hand pose dataset. \textbf{Bold} values stand for the top performer.}}
 \vskip -0.2cm
	\renewcommand{\arraystretch}{1.15} 
	\scalebox{0.73}{
		\begin{tabular}{c|c|c|c|c|c}
			\Xhline{1pt}
			{Method}&{Subject 1} &{Subject 2} &{Subject 3}&{Subject 4}&Time (s)\\ \cline{2-5}
			\Xhline{1pt}
			MR-RPM~\cite{ma2018nonrigid} &0.0940 &0.0834&0.1028&0.1388&0.2382\\
			BCPD~\cite{hirose2020bayesian}&0.1027 &0.1055  &0.1080 &0.1579 &0.6890\\
			
			GMM~\cite{jian2010robust}&0.0571 &0.0547&0.0734 &0.0917 &0.1140 \\
			ZAC~\cite{wang2020zero}&0.4886 &0.4566&0.4879 &0.4935 &0.4254 \\
			Ours&\textbf{0.0383} &\textbf{0.0481} &\textbf{0.0537} &\textbf{0.0879}&\textbf{0.1074}\\
			\Xhline{1pt}
		\end{tabular}
	}
	\label{tab:hand}
 \vskip -0.3cm
\end{table}

\begin{figure}
    \centering
\includegraphics[width=0.9\linewidth]{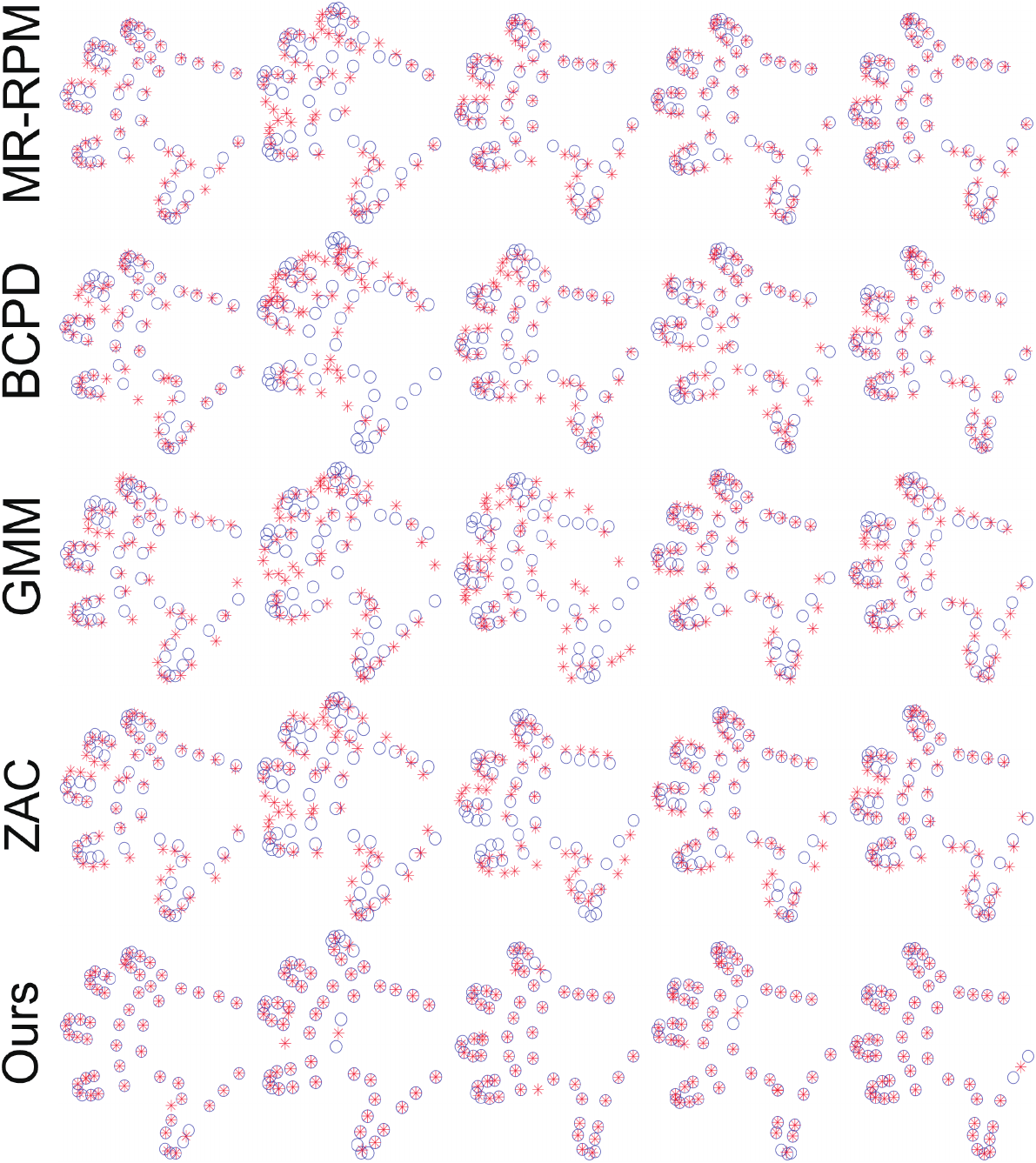}
     \vskip -0.2cm
     \caption{Qualitative comparisons on the 2D hand pose dataset.}\label{fig:hand_reg}
	\vskip -0.3cm
\end{figure}

\subsection{3D Non-Rigid Point Cloud Registration}
Since our method is dimension-independent, we further substantiate its efficacy on 3D point clouds and compare it with eight state-of-the-art 3D registration or deformation approaches, including BCPD~\cite{hirose2020bayesian}, GBCPD~\cite{hirose2020bayesian}, Fast\_RNRR~\cite{zhang2021fast}, AMM\_NRR~\cite{yao2023fast}, Sinkhorn~\cite{feydy2019interpolating}, as well as network-based ones Nerfies~\cite{park2021nerfies}, NDP~\cite{li2022non}, and NSFP~\cite{li2021neural}. For efficiency, we downsample the point clouds from datasets FAUST~\cite{CVPR2014} and TOSCA~\cite{bronstein2008numerical} using voxel grid filtering, with a point size of $3,000\sim4,000$. More experiments are presented in the \emph{Supplementary Material}.
\vskip -0.3cm
\begin{table*}[t]
	\centering
	\caption{Quantitative comparisons on the real-world FAUST human scan dataset.}
	\renewcommand{\arraystretch}{1} 
	\scalebox{0.8}{
		\begin{tabular}{c|c|c|c|c|c|c|c|c|c|c|c|c}
			\Xhline{1pt}
			\diagbox{Method}{Settings}&Intra-1&Intra-2&Intra-3&Intra-4&Intra-5&Intra-6&Inter-1&Inter-2&Inter-3&Inter-4&Average&Time (s) \\ \cline{1-13}
   BCPD~\cite{hirose2020bayesian}&0.0913&	0.1011&0.0872&0.0577&	0.1004	&0.0746&	0.1196&	0.0705&0.0935&0.0923&0.0888&3.0359\\
   GBCPD~\cite{hirose2022geodesic}&0.0285&0.0212&0.0211&0.0260&0.0244&0.0339&	0.0359	&0.0340&	0.0212&	0.0190&0.0265&1.9346\\
   Fast\_RNRR~\cite{Yao_2020_Fast_RNRR}&0.0430&0.0487&0.0397&0.0504&0.0429&0.0391&0.1358&0.0743&0.0477&0.0358&0.0557&\textbf{0.6324}\\
   AMM\_NRR~\cite{yao2023fast}&0.0544&0.0486&0.0400&0.0539&0.0405&0.0393&0.0838&0.0686&0.0422&0.0399&0.0511&2.0438\\
   Sinkhorn~\cite{feydy2019interpolating}&0.0654&	0.0638&	0.1372&	0.1096&	0.0749&	0.0821&	0.2467&	0.0781&	0.1400&	0.1720&0.1170&2.0377\\
   Nerfies~\cite{park2021nerfies}&0.0120 &0.0107 &0.0138 &0.0129 &0.0135&0.0118
&0.0121 &0.0144 &0.0140 &0.0140&0.0129&9.4287\\
   NDP~\cite{li2022non}&0.0183&	0.0199&	0.0192&	0.0152&	0.0170&0.0149&	0.0181&	0.0198&	0.0164&	0.0155&0.0174&1.7590\\
   NSFP~\cite{li2021neural}&0.0126 &0.0134 &0.0132 &0.0118 &0.0137 &0.0142&0.0167 &0.0162 &0.0148 &0.0166&0.0143&2.4607\\
  Ours&\textbf{0.0086}&	\textbf{0.0089}&	\textbf{0.0103}&	\textbf{0.0096}&	\textbf{0.0089}&	\textbf{0.0081}&	\textbf{0.0097}&	\textbf{0.0099}&	\textbf{0.0094}&	\textbf{0.0081}&\textbf{0.0092}&2.3746\\
			\Xhline{1pt}
		\end{tabular}
	}
	\label{tab:FAUST}
\end{table*}
\paragraph{Registration for real human scans.} ~\cref{tab:FAUST} and~\cref{fig:faust} report the quantitative and qualitative comparison results 
on the FAUST human dataset, respectively. The evaluation is conducted using six sets of subjects in six different and challenging poses for each subject. We first perform \emph{intra-class} registration, \ie, deforming the first human geometry to match the other poses for the same subject. Then, to validate the capability of the designed approach against large deformations, we  
further conduct an \emph{inter-class} registration test by aligning the first human pose of the $i$-th subject to all the poses of the $(i+2)$-th subject ($i=1, 2, 3, 4$). 
The statistical results summarized in \cref{tab:FAUST} demonstrate that our method achieves the highest registration accuracy across all subjects and outperforms competitors by a significant margin, even several orders of magnitude higher. Notably, while achieving remarkable accuracy, our method also maintains efficiency comparable to most competitors, making it a highly practical and effective solution. The qualitative comparison results in~\cref{fig:faust} indicate that our method not only ensures higher-quality deformations but also recovers the geometric details as well as the topology of the target subject more accurately and faithfully. 
\vskip -0.3cm

\begin{figure}[t]
    \centering
    \includegraphics[width=0.85\linewidth]{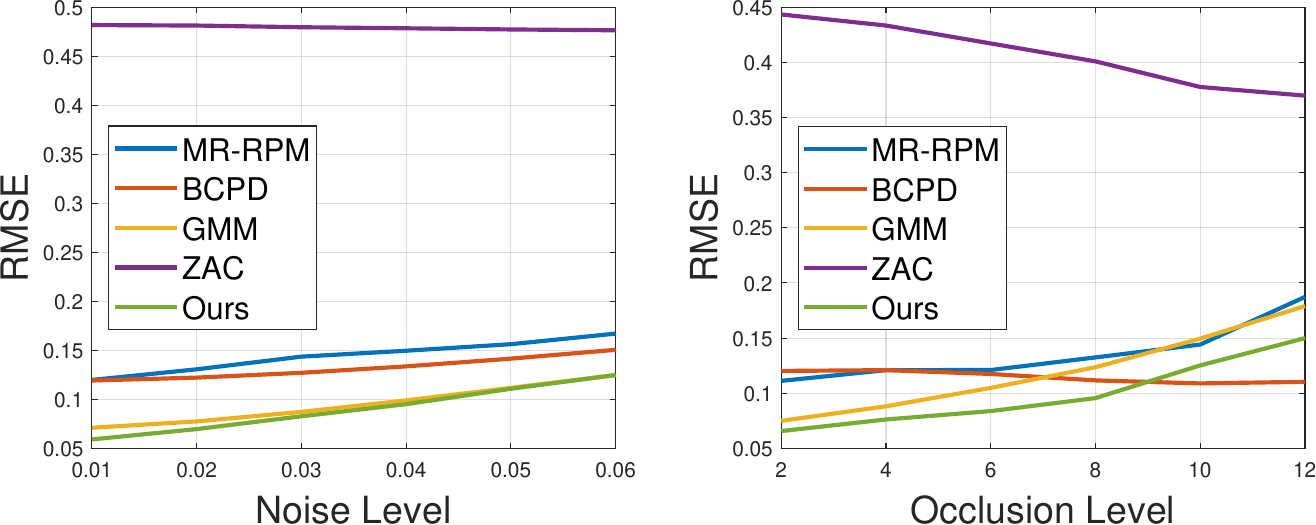}
    \vskip -0.2cm
    \caption{Robustness comparisons against external disturbances.}
    \label{fig:robustness}
    \vskip -0.5cm
\end{figure}

\begin{figure*}
    \centering
\includegraphics[width=0.85\linewidth]{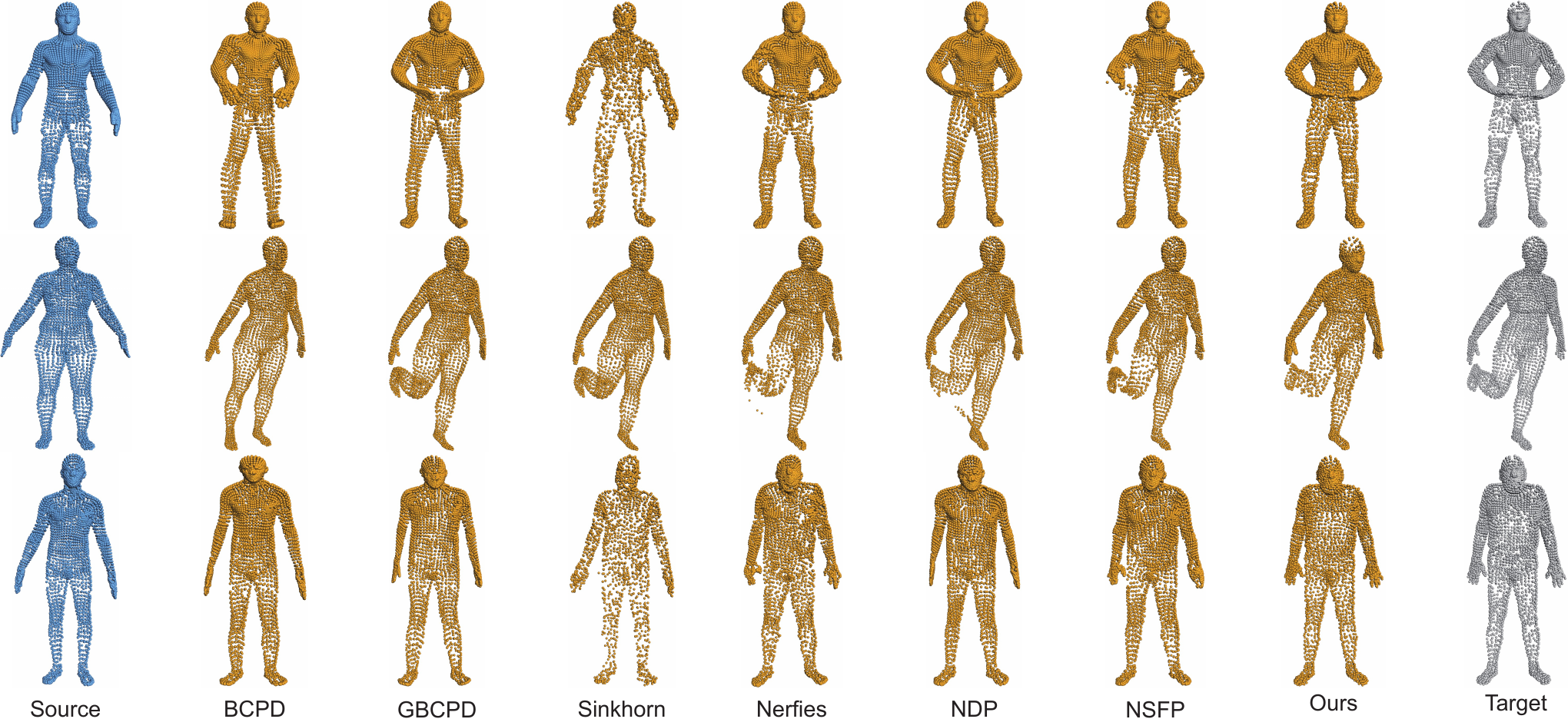}
 \vskip -0.2cm
 \caption{Qualitative comparisons on the FAUST dataset. The top and the bottom rows represent the test registrations of intra and inter-subject, respectively. }
 \label{fig:faust}
\vskip -0.5cm
\end{figure*}

\paragraph{Registration for larger deformations.} 
We further verify whether the proposed method improves the registration performance for point cloud pairs with much larger deformations. We evaluate four classes of animals from the TOSCA dataset~\cite{bronstein2008numerical} and report the average RMSE for each class of them. As illustrated in~\cref{fig:TOSCA}, the source and target point sets exhibit significant pose differences, making the registration quite challenging. \cref{tab:TOSCA} summarizes the quantitative comparison results. We exclude Fast\_RNRR, AMM\_NRR, and Sinkhorn from our analysis because they exhibit significant deviations from the target poses, rendering the error metrics unreliable. Our method consistently outperforms all the baselines by a large margin. \cref{fig:TOSCA} demonstrates that our method delivers highly stable and accurate registration results for point clouds with large deformations, even without the point-wise correspondences.

\begin{table}[t]
	\centering
	\caption{Quantitative comparisons on the TOSCA dataset.}
 \vskip -0.2cm
	\renewcommand{\arraystretch}{1} 
	\scalebox{0.8}{
		\begin{tabular}{c|c|c|c|c|c}
			\Xhline{1pt}
			{Method}&Cat&Centaur &Dog&Gorilla&Average\\ \cline{1-6}
BCPD~\cite{hirose2020bayesian}&3.9884&	8.1017&7.2800&5.6253&5.9935\\
GBCPD~\cite{hirose2022geodesic}&1.5631&2.9480&1.5300&3.5751&2.6523\\  Nerfies~\cite{park2021nerfies}&3.2704&2.8826&1.3612 &2.2809&2.3211\\ NDP~\cite{li2022non}&4.3639&3.4373&	3.1285&	2.8312&3.2560	\\
NSFP~\cite{li2021neural}&1.8774 &2.6425&1.6734 &2.2044 &2.0710\\
  Ours&\textbf{1.3496}&\textbf{1.8125}&	\textbf{1.2088}&	\textbf{1.6807}&\textbf{1.5247}\\
			\Xhline{1pt}
		\end{tabular}
	}
	\label{tab:TOSCA}
 \vskip -0.5cm
\end{table}

\begin{figure}
\centering
\includegraphics[width=\linewidth]{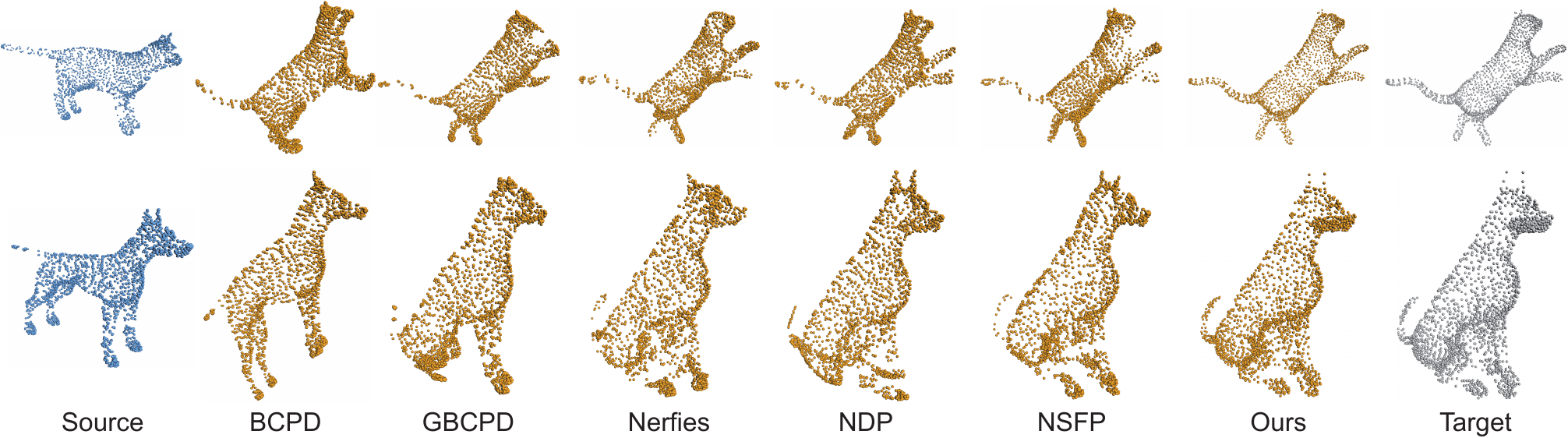}
\vskip -0.2cm
 \caption{Qualitative comparisons on the TOSCA dataset, where much larger shape deformations exist.}
\label{fig:TOSCA}
\vskip -0.4cm
\end{figure}

\subsection{Ablation Study}
\paragraph{Effect of the improved Nystr{\"o}m method.} \cref{fig:ablation_nystrom} reports the registration error, running time, and the matrix approximation error (defined in \cref{prop:1}), between the improved or clustered Nystr{\"o}m approximation method (Ours) and the random one on two randomly extracted FAUST models. We vary the approximation ratio $R\in[0.02, 0.4]$ with $\Delta R=0.02$. It can be seen that our method obtains significant registration and timing performance boost and decrease in matrix approximation error by a large margin, especially on lower ratios. 
\vskip -0.4cm
\begin{figure*}[t]
	\centering
\includegraphics[width=0.75\linewidth]{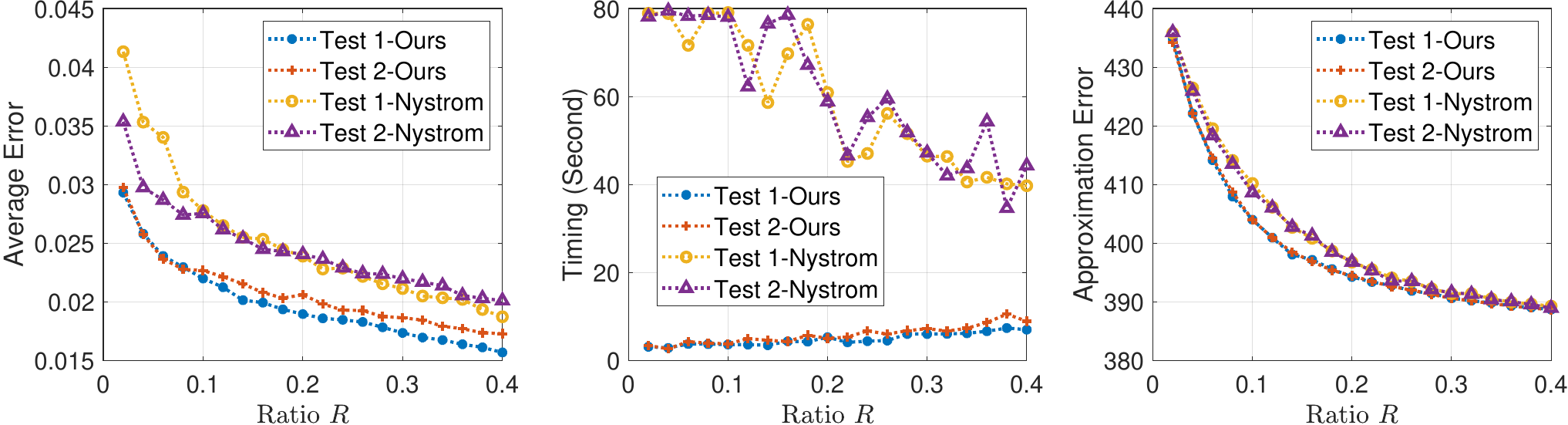}
\vskip -0.2cm
\caption{Comparisons between the  Nystr{\"o}m low-rank approximation and our clustering-improved one for non-rigid registration.
}\label{fig:ablation_nystrom}
\vskip -0.4cm
\end{figure*}

\begin{figure}
\centering
\includegraphics[width=0.45\linewidth]{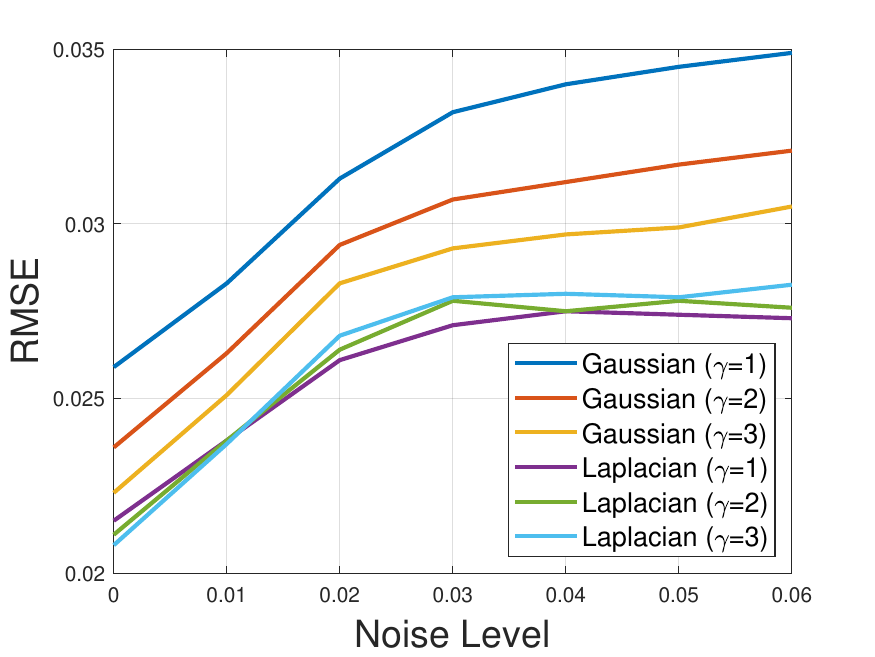} 
\includegraphics[width=0.425\linewidth]{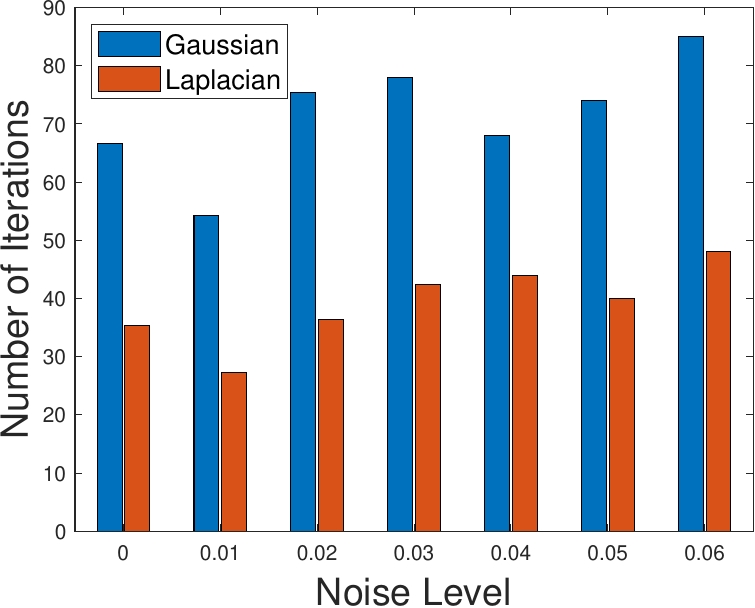} 
\vskip -0.3cm
\caption{Comparisons between the Gaussian and the Laplacian kernel functions. Left: Results with various kernel bandwidth $\gamma\in[1, 3]$. Right: The number of iterations when algorithms converge.} 
\label{fig:GaussVSLaplace}
\vskip -0.6cm
\end{figure}
\vskip -0.5cm
\paragraph{Laplacian VS. Gaussian.} 
\cref{fig:GaussVSLaplace} summarizes both the quantitative and qualitative comparison results between the kernel functions of Gaussian and Laplacian. We validate the merits and robustness of the Laplacian Kernel by aligning the source Bunny model~\cite{DataStanford1} contaminated by a set of noise ($\sigma\in[0, 0.06]$) to a randomly deformed Bunny. The kernel bandwidth $\gamma$ is varied in $[1, 3]$. We report the RMSE and the average number of iterations when algorithms converge in \cref{fig:GaussVSLaplace}.
~We find that the Laplacian kernel consistently outperforms the Gaussian kernel across all settings, suggesting the merits of the sparsity-induced  $\ell_1$ norm. Moreover, the Laplacian kernel delivers faster convergence and more accurate registration results (see the \emph{Supplementary Material}). 

\subsection{Applications}
\paragraph{Shape transfer.} 
As depicted in \cref{fig:shape_transfer}, we apply the proposed method to transfer shapes belonging to \emph{different categories} that require substantial deformations. We first transfer two geometries with the identical topology (sphere and cube) and then proceed to transfer CAD models from ShapeNet~\cite{chang2015shapenet}, which presents a more challenging task. Results indicate the effectiveness of our method in achieving accurate shape deformation while faithfully preserving the geometric details of the source shapes. Notably, our method consistently produces high-quality deformation results even when the shapes possess significantly distinct topology. More results on shape transfer are presented in the \emph{Supplementary Material}.
\vskip -0.3cm

\vskip -0.9cm
\paragraph{Medical registration.} Deforming a standard medical template  to match those captured from individual patients is a crucial step in the field of medical data analysis. In \cref{fig:medical}, we demonstrate the efficacy of our method by aligning a 3D inhale lung volume to two exhale lungs~\cite{castillo2013reference} and two brain vessels~\cite{DataBrain}, extracted from real-world CT and MRA images. Despite the presence of complex structures, large deformations, and mutual interference, our method consistently achieves impressive results in accurately deforming the template models to align the target shapes.

\begin{figure}
    \centering
\includegraphics[width=0.9\linewidth]{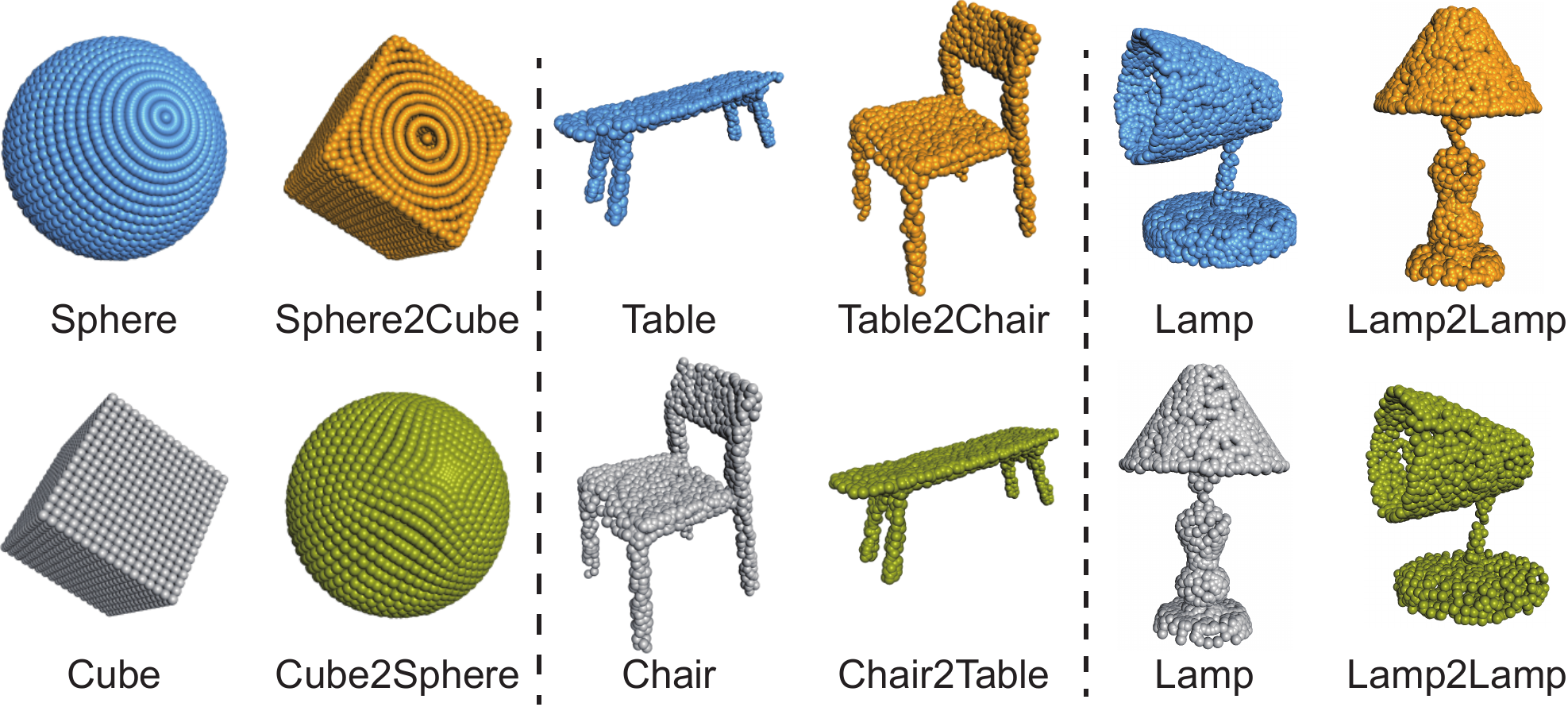}
\vskip -0.2cm
    \caption{Application of the proposed method to shape transfer.}
\label{fig:shape_transfer}
\vskip -0.3cm
\end{figure}

\begin{figure}
    \centering
\includegraphics[width=0.9\linewidth]{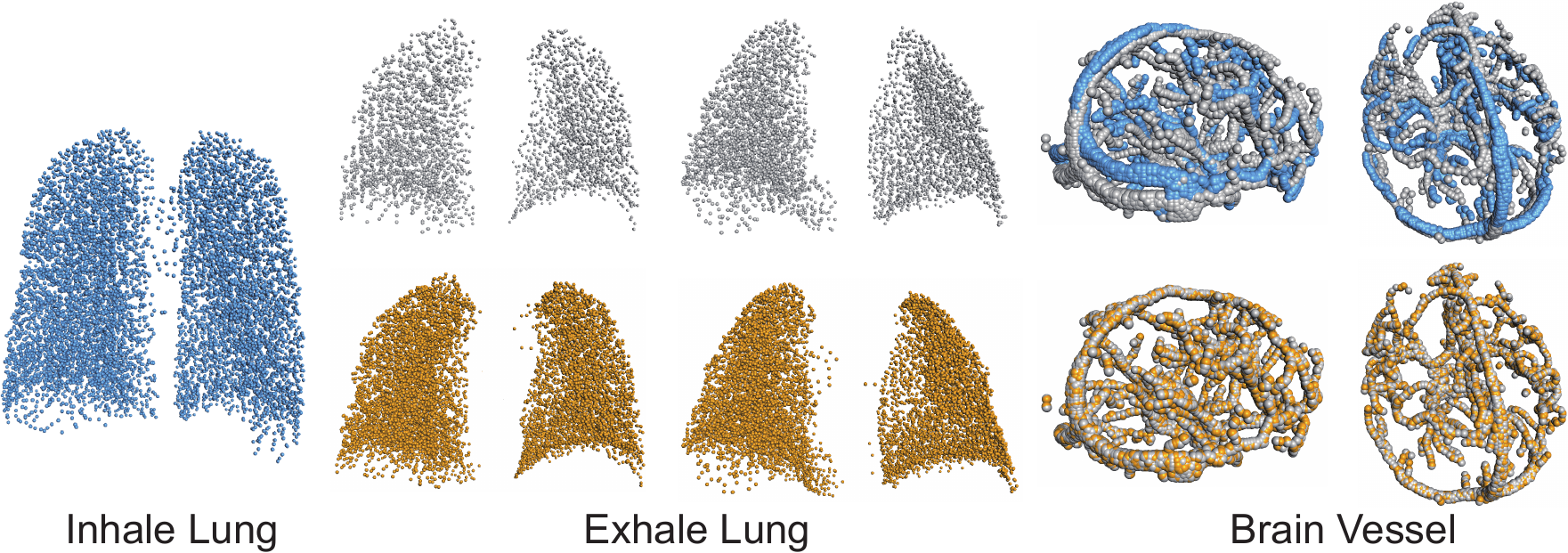}
\vskip -0.2cm
\caption{Application of our method to medical data registration.} 
\label{fig:medical}
\vskip -0.5cm
\end{figure}

\section{Conclusions}
We proposed an algorithm for solving non-rigid point set registration without prescribed correspondences. The key contribution of our method lies in reformulating non-rigid registration as an unsupervised clustering process that enables {holistic optimization}, dimension-independent, {closed-form solutions}, and {handling large deformations} simultaneously. Moreover, we introduce the $\ell_1$-induced Laplacian kernel to achieve a more robust solution than the Gaussian kernel and provide a rigorous approximation bound for the Nystr{\"o}m  method.
Our method achieves higher-quality results than traditional methods and recent network models, particularly on geometries that exhibit significant deformations.  We also showcase its applicability in challenging tasks such as shape transfer and medical registration. 

\noindent{\textbf{Acknowledgements.}} This work is partially funded by the Strategic Priority Research Program of the Chinese Academy of Sciences (XDB0640000), National Science and Technology Major Project (2022ZD0116305), National Natural Science Foundation of China (62172415,62272277,62376267), and the innoHK project.

{\small
\bibliographystyle{ieeenat_fullname}
\bibliography{references_forever}
}

\end{document}